\documentclass{article}

\usepackage{microtype}
\usepackage{graphicx}
\usepackage{subfigure}
\usepackage{booktabs} 
\usepackage{fancyhdr} 
\usepackage{lastpage} 
\usepackage{extramarks} 
\usepackage{graphicx} 
\usepackage{lipsum} 
\usepackage{bm}
\usepackage{bbm}
\usepackage{mathtools}
\usepackage{amssymb}
\usepackage{amsthm}
\usepackage{hyperref}
\DeclareMathAlphabet{\pazocal}{OMS}{zplm}{m}{n}
\newcommand{\theHalgorithm}{\arabic{algorithm}}

\newcommand{\A}{\pazocal{A}}
\newcommand{\G}{\pazocal{G}}
\newcommand{\X}{\pazocal{X}}
\newcommand{\Y}{\pazocal{Y}}
\newcommand{\Z}{\pazocal{Z}}
\newcommand{\M}{\mathcal{M}}
\newcommand{\N}{\mathcal{N}}
\newcommand{\MM}{\pazocal{M}}
\newcommand{\RR}{\pazocal{R}}
\newcommand{\PP}{\mathcal{P}}
\newcommand{\B}{\pazocal{B}}
\newcommand{\F}{\pazocal{F}}
\newcommand{\K}{\pazocal{K}}
\newcommand{\Hh}{\pazocal{H}}
\newcommand{\Sp}{\pazocal{S}}

\newtheorem{theorem}{Theorem}[section]
\newtheorem{lemma}[theorem]{Lemma}

\newtheorem{proposition}[theorem]{Proposition}

\newtheorem{assumption}[theorem]{Assumption}
\theoremstyle{definition}
\newtheorem{definition}[theorem]{Definition}

\DeclareMathOperator{\pa}{pa}
\DeclareMathOperator{\neigh}{ne}
\DeclareMathOperator{\an}{an}

\makeatletter
\newcommand*{\indep}{%
  \mathbin{%
    \mathpalette{\@indep}{}%
  }%
}
\newcommand*{\nindep}{%
  \mathbin{
    \mathpalette{\@indep}{\not}
  }%
}
\newcommand*{\@indep}[2]{%
  \sbox0{$#1\perp\m@th$}
  \sbox2{$#1=$}
  \sbox4{$#1\vcenter{}$}
  \rlap{\copy0}
  \dimen@=\dimexpr\ht2-\ht4-.2pt\relax
  \kern\dimen@
  {#2}%
  \kern\dimen@
  \copy0 
} 
\makeatother

\usepackage[accepted]{icml2018}

\icmltitlerunning{Multi-Domain Translation by Learning Uncoupled Autoencoders}

\begin{document}

\twocolumn[
\icmltitle{                                                                                                                                                                                                                                                                                                                                                                Multi-Domain Translation by Learning Uncoupled Autoencoders
}
\icmlsetsymbol{equal}{*}

\begin{icmlauthorlist}
\icmlauthor{Karren D. Yang}{mit1,mit2}
\icmlauthor{Caroline Uhler}{mit1}
\end{icmlauthorlist}

\icmlaffiliation{mit1}{Laboratory for Information \& Decision Systems, and 
Institute for Data, Systems, and Society, 
Massachusetts Institute of Technology}
\icmlaffiliation{mit2}{Department of Biological Engineering, Massachusetts Institute of Technology}

\icmlcorrespondingauthor{Karren Yang}{karren@mit.edu}
\icmlcorrespondingauthor{Caroline Uhler}{cuhler@mit.edu}

\icmlkeywords{boring formatting information, machine learning, ICML}

\vskip 0.3in
]



\printAffiliationsAndNotice{}  

\begin{abstract}
Multi-domain translation seeks to learn a probabilistic coupling between marginal distributions that reflects the correspondence between different domains. We assume that data from different domains are generated from a shared latent representation based on a structural equation model. Under this assumption, we prove that the problem of computing a probabilistic coupling between marginals is equivalent to learning multiple uncoupled autoencoders that embed to a given shared latent distribution. In addition, we propose a new framework and algorithm for multi-domain translation based on learning the shared latent distribution and training autoencoders under distributional constraints. A key practical advantage of our framework is that new autoencoders (i.e., new domains) can be added sequentially to the model without retraining on the other domains, which we demonstrate experimentally on image as well as genomics datasets.
\end{abstract}

\section{Introduction}

Unsupervised translation between multiple domains is becoming increasingly popular in fields such as computer vision \cite{zhu2017unpaired} and computational biology \cite{mcdermott2018semi}. In these problems, one often has access to large quantities of unpaired data from different domains, and the objective is to learn a probabilistic coupling between the observed marginal distributions that reflects the correspondance between the domains.

We consider unsupervised multi-domain translation under the assumption of a shared latent representation of the different domains. This problem arises in settings where the different domains have some common structure or latent generative process. In computer vision, for instance when working with faces, one might consider the facial shape and attributes to be a higher level structure conserved across different populations (e.g. male, female, old, young.) while differences in other characteristics such as hair color are mutable between domains. Similarly, in genomics, one might consider relating different types of experimental data that are generated from the same latent representation of a cell population. A key difference between the image-to-image translation problem and translation between biological data modalities is that correspondence between different biological data usually cannot be enforced by the neural architecture such as through convolutions. Thus strategies for unsupervised translation of images based on sharing or transferring weights of neural networks cannot easily be translated to biological problems.

Many prominent methods for unsupervised translation including \cite{almahairi2018augmented, liu2017unsupervised, yi2017dualgan, kim2017learning} are based on the CycleGAN framework \cite{zhu2017unpaired}, which uses generative adversarial networks \cite{goodfellow2014generative}. The main idea is to train two generative networks to transport images between two domains such that (i) adversarial networks in both domains cannot discriminate between real and transported data, and (2) the translations are consistent with each other when composed, i.e. translating from \emph{domain 1} $\rightarrow$ \emph{domain 2} $\rightarrow$ \emph{domain 1} recovers the original data. However, CycleGAN does not use shared latent structure and only considers transport between one pair of domains at a time. More recently, \citet{liu2017unsupervised} considered image-to-image translation under the assumption of a shared latent representation by composing the GAN objective with two variational autoencoders that share weights. The model is based on the CycleGAN framework and performs transport between pairs of domains at a time, which limits scalability to multiple domains. 

\begin{figure*}[h]
\centering
\subfigure[]{\includegraphics[scale=.3]{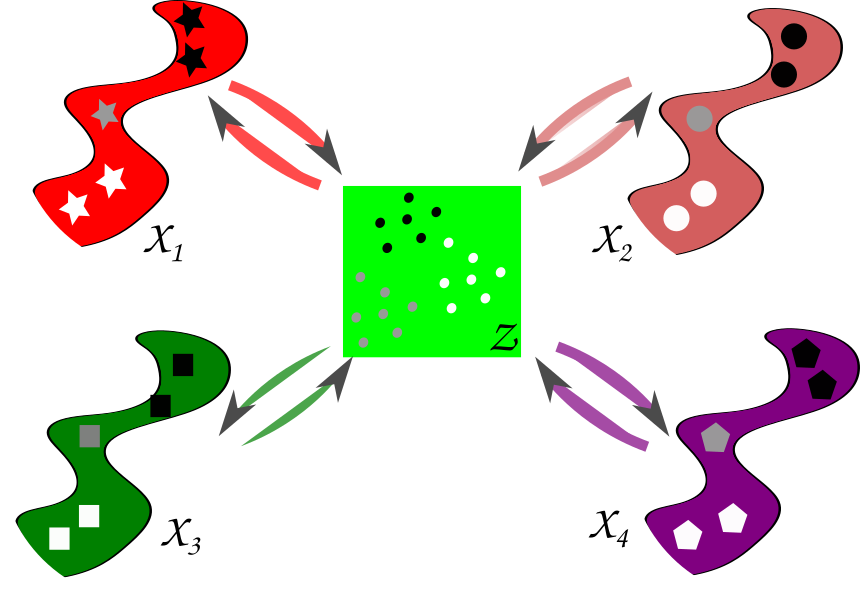}}\hspace{0.2cm}
\subfigure[]{\includegraphics[scale=.3]{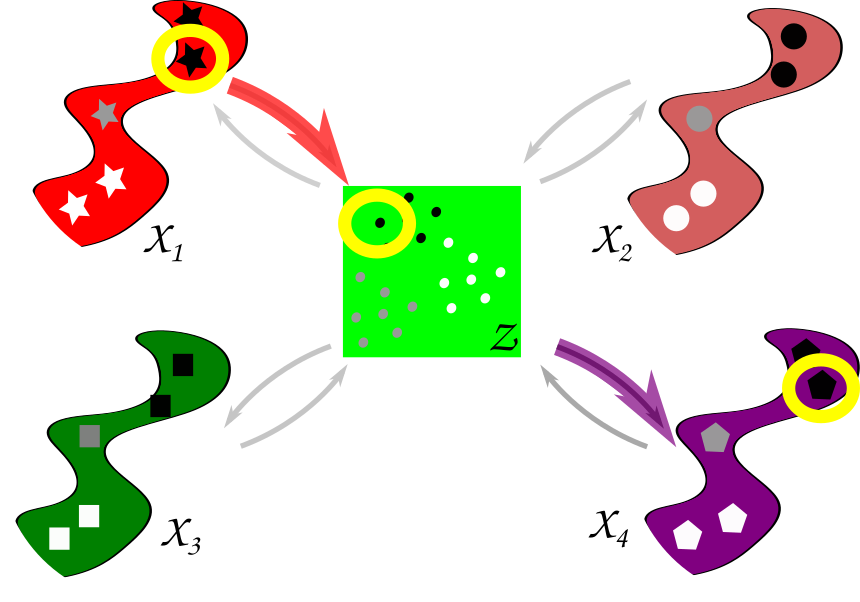}}
\vspace{-0.2cm}
\caption{ (a) The assumption that the different data domains share a latent representation is leveraged to decompose the multi-domain translation model between multiple domains into multiple autoencoders that can be trained separately. (b) To transport data between domains, we compose the encoder (red) of the source domain with the decoder (purple) of the target domain. 
}
\label{fig:intro}
\end{figure*}

Here we propose a novel framework for learning transport maps between multiple domains. The main idea is to leverage the assumption of a shared latent representation in order to decompose the multi-domain translation model between $k$ domains into $k$ \emph{uncoupled} autoencoders (Figure \ref{fig:intro}a). These autoencoders are trained separately and composed to translate between different pairs of domains (Figure \ref{fig:intro}b). This offers the following advantages:

\begin{itemize}
\itemsep0em
\item Modularity: $\binom{k}{2}$ pairwise transport maps between $k$ domains can be constructed from $k$ autoencoders.
\item Flexibility: since the $k$ autoencoders can be trained separately, the training procedure is more flexible -- for example, new domains can be added later without needing to retrain the model on the original domains.
\item Efficiency: while our framework uses adversarial training to match the data to a common latent space, the discriminators operate in the latent space and thus can have much simpler architectures as compared to models with discriminators in the original domain spaces.
\item Finally, our framework is general in the sense that it is not specifically designed for image-to-image translation. Autoencoders can be designed for transport between very different domains. We demonstrate this by applying our model to imaging and genomics data.
\end{itemize}


The paper is structured as follows. We begin by formalizing our assumption of a latent representation using structural equation models (Section \ref{sec:prelim}). Subsequently, we present our framework for multi-domain translation using uncoupled autoencoders (Section \ref{sec:main}), providing theoretical justification for our approach as well as practical algorithms. Finally, in Section \ref{sec:exp}, we demonstrate the efficacy of our algorithms on handwritten digits, CelebA faces and genomics datasets.

\section{Preliminaries and Related Work} \label{sec:prelim}

\subsection{Notation} Let $\X_1, \X_2, \cdots \X_k$ denote the data domains of interest, and let $\Z$ denote the domain of their shared latent representation. For most practical purposes it suffices to work in the reals; so we assume throughout that $\X_i\subseteq \mathbb{R}^{n_i}$ for all $i\in[k]:=\{1,\dots , k\}$ and $\Z=\mathbb{R}^d$. 
For any domain $\X$, we let $X$ denote a random variable in this domain with distribution $P_X$ and corresponding density function $p_X$. Finally we let $\bf{X}$ denote a tuple of variables $(X_1,\cdots,X_k)$. 

\subsection{Problem Setting} We consider the problem of unsupervised translation between multiple domains: given random variables $X_1, X_2, \cdots, X_k$ with marginal distributions $P_{X_1}, P_{X_2}, \cdots P_{X_k}$, learn a probabilistic coupling
 $P_{\bf{X}}$ between the marginals that reflects the correspondance between the variables. This can be formulated alternatively as a \emph{conditional generative modeling task}, in which the objective is to estimate the pairwise conditional distributions 
 $Q_{X_i|X_j}$ and $Q_{X_j|X_i}$ with density functions $q_{X_i|X_j}$ and $q_{X_j|X_i}$ under the constraint that they are consistent with the same probabilistic coupling for all $i, j \in [k]$. In the $k=2$ setting, for example, this would boil down to requiring that
\begin{align*}
p_{X_1, X_2}(x_1, x_2) &= q_{X_2|X_1}(x_2|x_1) p_{X_1}(x_1)\\
&= q_{X_1|X_2}(x_1|x_2) p_{X_2}(x_2)
\end{align*}
for all $x_1 \in \X_1,x_2 \in \X_2$.

\begin{figure}
\centering
\includegraphics[scale=.35]{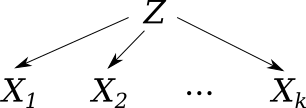}
\caption{Graphical depiction of the structural equation model underlying the probabilistic coupling $P_{\bf{X}}$.}
\label{fig:sem}
\end{figure}

There are various ways of defining a probabilistic coupling between $P_{X_1}, P_{X_2}, \cdots P_{X_k}$. In this work, we assume that the random variables $X_1, X_2, \cdots, X_k$ are generated independently from a common latent variable $Z$ as shown in Figure \ref{fig:sem}. Specifically:

\begin{assumption} \label{ass:sem} The random variables $X_1, \dots, X_k$ are generated by the following (causal) mechanism:
\begin{equation*} \label{eq:sem}
X_i = f_i(Z, N_i), \quad \forall i \in [k],
\end{equation*}
where $f_i$, $i\in[k]$ are injective functions, $Z$ is a latent variable with distribution $P_Z$, and $N_i$, $i\in[k]$, are independent noise variables with distribution $P_{N_i}$. 
\end{assumption}

This model is a \emph{structural equation model} and implies the following factorization of the joint distribution $P_{\bf X}$:
\begin{equation}
p_{{\bf{X}}} ({\bf{x}}) = \int_\Z \prod_{i=1}^k p_{X_i | Z} (x_i | z) p_Z(z) dz,
\end{equation}
where $p_Z$ is the probability density of $Z$, and $p_{X_i | Z}$ is the conditional distribution of $X_i$ given $Z$ that reflects the generative process. Note that the relationship between $X_i$ and $Z$ is allowed to be stochastic due to the noise variables in the structural equation model.

\subsection{Related Work}

{\bf CycleGAN.} Unsupervised transport between domains has been addressed by many others before us, particularly in the context of unsupervised image-to-image translation. To our knowledge, many prevalent approaches \cite{almahairi2018augmented, liu2017unsupervised, yi2017dualgan, kim2017learning} are based on the CycleGAN framework \cite{zhu2017unpaired}, which has proven to be very successful at training generative adversarial networks to translate between two unpaired domains. Specifically, \citet{zhu2017unpaired} proposed to represent the conditional distributions $Q_{X_1|X_2}$ and $Q_{X_2|X_1}$ between two domains with deterministic transport maps $\G^{21}: \X_2 \rightarrow \X_1$ and $\G^{12}: \X_1 \rightarrow \X_2$ parameterized by convolutional neural networks. These transport maps are trained to satisfy the following constraints:

\begin{enumerate}
\item The distribution of images $P_{X_1}$ under the transformation $\G^{12}$ must be indistinguishable from the distribution $P_{X_2}$, i.e. $\G^{12} \# P_{X_1} = P_{X_2}$. Similarly, the distribution of images $P_{X_2}$ under the mapping $\G^{21}$ must be indistinguishable from the distribution $P_{X_1}$, i.e. $\G^{21} \# P_{X_2} = P_{X_1}$. 
\item Composing the transport maps must result in the identity function, i.e., $\G^{21} \circ \G^{12}(x) = x$ for all $x \in \X_1$, and $\G^{12} \circ \G^{21}(x) = x$ for all $x \in \X_2$.
\end{enumerate}

While CycleGAN has seen great success in unsupervised image-to-image translation, it does not make use of a shared latent structure and only considers transport between one pair of domains at a time. More recently, \citet{choi2018stargan} proposed StarGAN to extend CycleGAN to the multi-domain setting by training a single generator to generate images from multiple domains. This model differs from ours as it does not use the assumption of a shared latent space; hence the model must be trained on all domains at the same time and applies primarily to data types that can use a common generator and discriminator for different domains.

{\bf Latent Space Assumption.} More closely related to our work are several papers that propose methods for unsupervised translation between different image domains under the assumption of a common latent space. For example, Coupled GAN \cite{liu2016coupled} trains two GANs with shared weights to learn a common representation of two domains. However, Coupled GAN is not designed for inference, and the training of networks is coupled, which is in contrast to our model where training can be uncoupled. In addition, \citet{liu2017unsupervised} composed the CycleGAN framework with variational autoencoders to concurrently learn the latent space representation as well as transport maps between two domains. However, the training of the two variational autoencoders is coupled due to weight-sharing and the CycleGAN objective, whch is in contrast to our model where training can be uncoupled over multiple autoencoders.

{\bf Regularized Autoencoders.} There are several forms of regularized autoencoders that are designed for generative modeling by matching distributions in the latent space. Variational autoencoders \cite{kingma2013auto} minimize the KL-divergence between generated and real images by maximizing the evidence lower bound. However, this framework assumes that the latent variable follows a Gaussian distribution. In our work we use a form of the Wasserstein autoencoder (WAE) \cite{tolstikhin2017wasserstein}, which is a theoretically motivated generalization of the adversarial autoencoder \cite{makhzani2015adversarial}. Both the adversarial autoencoder and the discriminator-based WAE train an adversary in the latent space to match encoded data samples to the real latent distributions. The recent Sinkhorn autoencoder uses the Sinkhorn divergence instead of discriminative loss to perform distribution matching in the latent space \cite{patrini2018sinkhorn}. However, none of these works address multi-domain translation, which our framework will address both in theory and in practice.

\section{Proposed Framework} \label{sec:main}

In this section, we present our new framework for performing unsupervised translation between multiple domains under the assumption that their coupling is given by the structural equation model in Assumption \ref{eq:sem}. We provide our theoretical results in the setting where the distribution $P_Z$ is known in Section~\ref{sec_PZ_known} and then provide algorithms for the general setting in Section~\ref{sec_PZ_unknown}.

\subsection{Multi-domain transport with known $P_Z$}
\label{sec_PZ_known}

\begin{figure}[!t]
\centering
\includegraphics[scale=.3]{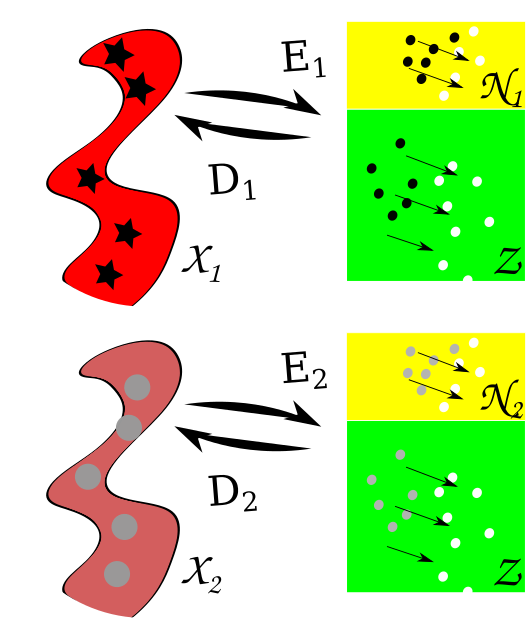}
\caption{Schematic of training uncoupled regularized autoencoders: the encoder is forced to match the encoded data distribution to the latent distribution; at the same time, both the encoder and the decoder must reconstruct samples from the data distribution.}
\label{fig:AE}
\end{figure}

In this section, we show that when the latent distribution $P_Z$ is known, then learning a probabilistic coupling between the marginals $P_{X_1}, P_{X_2}, \cdots, P_{X_k}$ under Assumption \ref{ass:sem} can be solved by learning multiple \mbox{\emph{\bf{uncoupled}}} autoencoders (see Figure \ref{fig:AE}
). Specifically, for each domain $i \in [k]$, we propose training a regularized encoder-decoder pair $(E_i:\X \rightarrow \Z \times \N_i, \;D_i: \Z \times \N_i \rightarrow \X)$ to minimize:
\begin{align} \label{eq:AE}
\mathbb{E}_{x \sim P_{X_i}} \left[ L_1(x, D_i\circ E_i(x)) + \lambda L_2(E_i \# P_{X_i} | P_{Z, N_i}) \right],
\end{align}
where $\lambda>0$ is a hyperparameter, $L_1$ is the Euclidean metric, and $L_2$ represents a divergence between probability distributions. Concretely, $L_1$ penalizes the \emph{reconstruction loss} of the autoencoder while $L_2$ penalizes the \emph{divergence} in the latent space between $E_i \# P_{\X_i}$, the encoded distribution of $P_{\X_i}$, and the latent generating distribution $P_{Z, N_i}:=P_{Z} \otimes P_{N_i}$. 
Translation from domain $i$ to $j$ is accomplished by composing the encoder from the source domain with the decoder from the target domain, i.e. taking
\begin{equation} \label{eq:x_trans}
X_{i \rightarrow j} := D_j(\pi^Z(E_i(X_i)), N_j)
\end{equation}
where $ X_i \sim P_{X_i}$, $N_j \sim P_{N_j}$, and the projection $\pi^Z: (z, n) \mapsto z$ restricts the output of the encoder to the domain of $Z$; see Figure~\ref{fig:AE_trans} for an illustration. Note that the role of the noise variables is to introduce \emph{stochasticity} into the mapping, i.e. our framework can handle stochastic mappings similar to \citet{almahairi2018augmented}.

\begin{figure}[h]
\centering
\includegraphics[scale=.3]{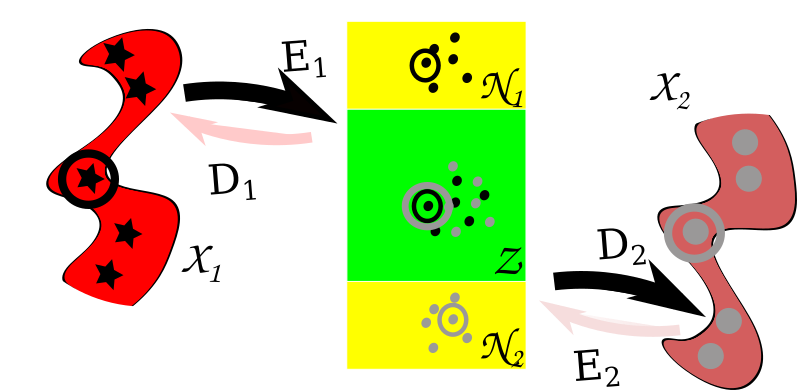}
\caption{Schematic of performing transport from domain $1$ to domain $2$: first $E_1$ is used to encode a sample from domain $1$ to a latent variable; subsequently $D_2$ is used to decode the latent variable to domain $2$.}
\label{fig:AE_trans}
\end{figure}

{\bf Implementation.} We parameterize $(E_i, D_i)$ using neural networks and minimize the objective function in (\ref{eq:AE})  via stochastic gradient updates as shown in Algorithm~\ref{alg:AE}. In particular, we choose $L_2$ to be the discriminative loss,
\begin{equation*}
L_2(P|Q) := \max_f \mathbb{E}_{x \sim P} \log f(x) + \mathbb{E}_{x \sim Q} \log (1-f(x)).
\end{equation*}
The resulting objective function is equivalent to the adversarial autoencoder \cite{makhzani2015adversarial} and the GAN-based WAE \cite{tolstikhin2017wasserstein}, which are methods for generative modeling based on autoencoding. Unlike those works, the ultimate goal of training the autoencoders for us is to compose them for multi-domain translation. 

\begin{algorithm}[H]
\begin{algorithmic}
   \STATE {\bfseries Input:} Distributions $P_Z, \{P_{X_i}\}_{i \in [k]}, \{P_{N_i} \}_{i \in [k]}$, initialized autoencoders $(E_{\theta_i}, D_{\phi_i})_{i \in [k]}$, discriminator $f_\omega$
   \STATE {\bfseries Output:} Updated autoencoders $(E_{\theta_i}, D_{\phi_i})_{i \in [k]}$
   \FOR{$i \in [k]$}
      \WHILE{$(E_{\theta_i}, D_{\phi_i})$ not converged}
      \STATE Sample $x_1, \cdots, x_N$ from $P_{X_i}$
      \STATE Sample $z_1, \cdots, z_N$ from $P_Z$
      \STATE Sample $n_1, \cdots, n_N$ from $P_{N_i}$

      \STATE Update $E_{\theta_i}, D_{\phi_i}$ by gradient descent on%
      \[ \frac{1}{N} \sum_{j=1}^N ||x_j - D_{\phi_i}(E_{\theta_i}(x_j)||_2^2 + \lambda \log f_\omega(E_{\theta_i}(x_j))\]%
      \STATE Update $f_\omega$ by gradient ascent on%
      \[\frac{1}{N} \sum_{j=1}^N \left[\log f_\omega(E_{\theta_i}(x_j)) + \log (1-f_\omega(z_j, n_j)) \right]\]%
      \ENDWHILE
   \ENDFOR
\end{algorithmic}
\caption{Training of autoencoders for multi-domain transport}
	\label{alg:AE}
\end{algorithm}


{\bf Theoretical properties.} To justify our approach, we show that the optimal solutions to (\ref{eq:AE}), which are tuples of autoencoders $(E_i, D_i)_{i \in [k]}$ over every domain, satisfy the properties of \emph{consistency} and \emph{completeness} under Assumption~\ref{ass:sem}. Before formalizing these two concepts, we introduce some notation and definitions. Let $Q_{Z|X_i}$ and $Q_{X_i|Z}$ be the conditional probability distributions induced by the encoder $E_i$ and decoder $D_i$ respectively, i.e. $\forall i \in [k]$, the corresponding densities are given by,
\[
q_{X_i|Z}(x|z) := \int_{\N_i} \delta_{D_i(z,n) = x} ~dP_{N_i}(n) \quad \forall z \in \Z
\]
and
\[
q_{Z|X_i}(z|x) := \delta_{\pi^Z(E_i(x)) = z}  \quad \forall x \in \X_i,
\]
where $\delta$ is the Dirac delta. Subsequently, we can define the conditional probability distribution $Q_{X_j | X_i}$ induced by composing $E_i$ and $D_j$ for any $i, j \in [k]$ as follows:
\[
q_{X_j|X_i}(x'|x) := \int_{\Z} q_{X_j|Z}(x'|z) q_{Z|X_i}(z|x) dz \quad \forall x \in \X_i.
\]
With this notation we discuss two notions of consistency.

\begin{definition}[Path consistency]
Let $(i_1, \cdots, i_{\ell})$ denote a sequence of domains. A tuple of autoencoders $(E_i, D_i)_{i \in [k]}$ is \emph{path-consistent} if for every sequence $(i_1, \cdots, i_{\ell})$,
\begin{align*}
\int_{\X_{i_{\ell-1}}} ... \int_{\X_{i_2}}
&\prod_{j'=1}^{\ell-1} q_{X_{i_{j'+1}}|X_{i_{j'}}}(x_{j'+1}|x_{j'})
 dx_2...dx_{\ell-1} \\
&= q_{X_{i_\ell}|X_{i_1}}(x_{\ell}|x_1). 
\end{align*}
\end{definition}
Path-consistency implies that any sequence of encodings and decodings starting in domain $i_1$ and ending in domain $i_{\ell}$ induces the same conditional distribution.

\begin{definition}[Global Consistency]
Let $Q^{(i)}$ be the joint distribution over $X_1, \cdots, X_k$ with density given by
\begin{align*}
q^{(i)}(\bm{x}) := \int_{\Z} \prod_{j \neq i} q_{X_j|Z}(x_j|z)q_{Z|X_i}(z|x_i) p_{X_i}(x_i) dz.
\end{align*}
A tuple of autoencoders $(E_i, D_i)_{i \in [k]}$ satisfies \emph{global consistency} if $Q^{(1)} = Q^{(2)} = \cdots = Q^{(k)}$. 
\end{definition}

Intuitively, global consistency means that the joint probability distribution generated by encoding $X_i$ using $E_i$ and decoding the resulting latent variable to all other domains ${j \in [k], j \neq i}$ using $D_j$ is equivalent for any source domain $i \in [k]$. These notions of consistency are generalizations of cycle-consistency \cite{zhu2017unpaired} to the probabilistic and multi-domain settings. The following proposition, whose proof is given in the Supplementary Material, states that our objective in (\ref{eq:AE}) satisfies these consistency properties.

\begin{proposition} \label{the:sym}
Under Assumption \ref{ass:sem}, every optimal solution of the objective in (\ref{eq:AE}) satisfies path and global consistency.
\end{proposition}

In general, the optimal solution of the objective in  (\ref{eq:AE}) is not unique because there can potentially be multiple ways to map between $Z$ and $X_i$. However, we can guarantee that the true probabilistic coupling under Assumption \ref{ass:sem} is represented by some solution in the optimal set. This is formalized in the following proposition, whose proof can again be found in the Supplementary Material.

\begin{proposition}[Completeness] \label{the:exp}
For any probabilistic coupling $P_{{\bf X}}$ satisfying Assumption~\ref{ass:sem}, there exists a tuple of autoencoders $(E_i, D_i)_{i \in [k]}$ optimizing (\ref{eq:AE}) such that $P_{\bf X} = Q_i\,$ for all $i \in [k]$, where $Q_i$ is defined as in Proposition \ref{the:sym}.
\end{proposition}

The consistency and completeness properties hold generally in the non-parametric setting, i.e.~when the autoencoders that we optimize over have sufficiently large capacity to fit the data distributions and optimize the objective function in (\ref{eq:AE}) to zero. In practice, however, the autoencoders are constrained by parameterization using neural networks and the objective function may not be optimized perfectly. The following result provides a bound on transport error in terms of autoencoder reconstruction and quality of latent encodings, and applies to the case where the class of autoencoders does not fit the latent and data distributions perfectly.


\begin{proposition} Let $Q_{X_{i \rightarrow j}}$ denote the distribution of $X_{i \rightarrow j}$ as defined in (\ref{eq:x_trans}). If the decoder $D_j$ is $\gamma_j$-Lipschitz, then the 1-Wasserstein distance $W(Q_{X_{i \rightarrow j}}, P_{X_j})$ satisfies
\begin{align*}
W(Q_{X_{i \rightarrow j}}, P_{X_j}) &\leq \gamma_j W(E_i \# P_{X_i}, P_Z \otimes P_{N_i})\\
 &\quad+ \gamma_j W(P_Z \otimes P_{N_j}, E_j \# P_{X_j}) \\
 &\quad+ \mathbb{E}_{x \sim P_{X_j}} L_1(x, D_j\circ E_j(x)).
\end{align*}
\end{proposition}

The first two terms in this bound indicate how far the encoded distributions $P_{X_i}$ and $P_{X_j}$ are from the latent distribution $P_{Z}$; the last term is the reconstruction loss of the autoencoder for domain $j$. For a proof, see the Supplementary Material.

\subsection{Multi-domain transport with unknown $P_Z$}
\label{sec_PZ_unknown}

\begin{figure*}[h]
\centering
\subfigure[]{\includegraphics[scale=.2]{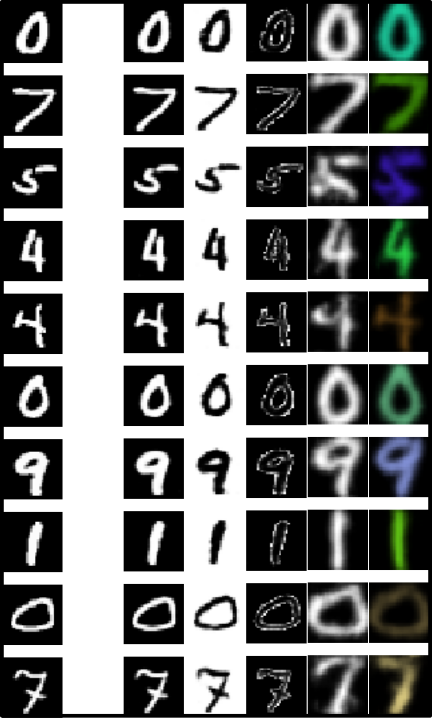}} \hspace{0.5cm}
\subfigure[]{\includegraphics[scale=.2]{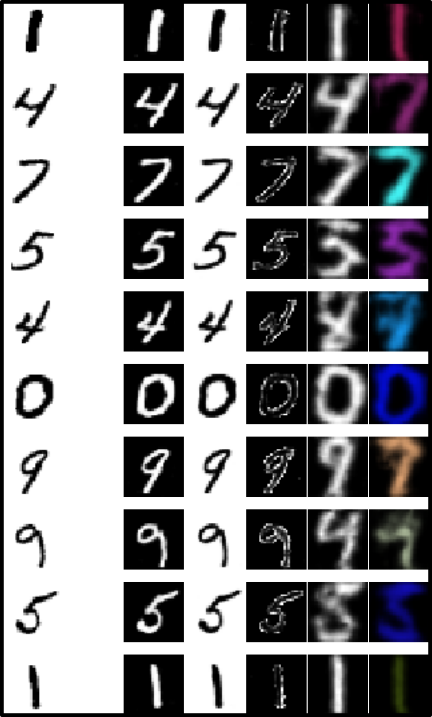}} \hspace{0.5cm}
\subfigure[]{\includegraphics[scale=.2]{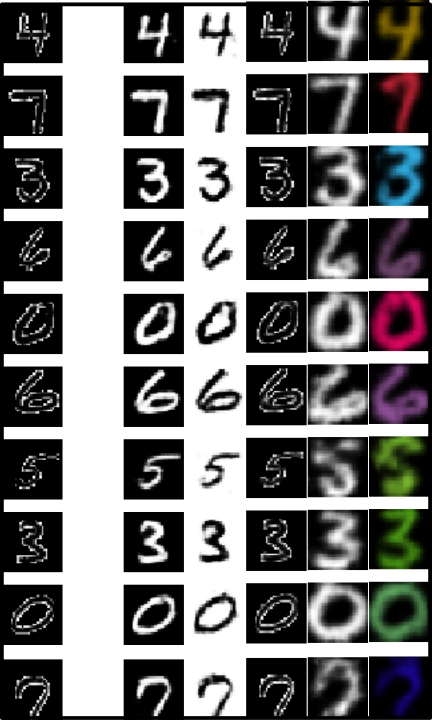}} \hspace{0.5cm}
\subfigure[]{\includegraphics[scale=.2]{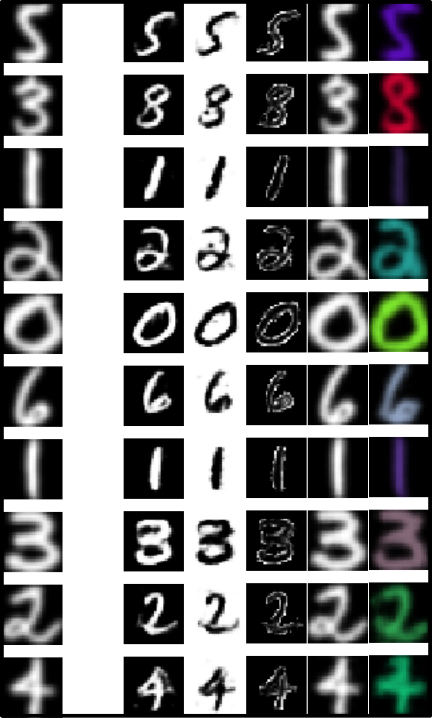}} \hspace{0.5cm}
\subfigure[]{\includegraphics[scale=.2]{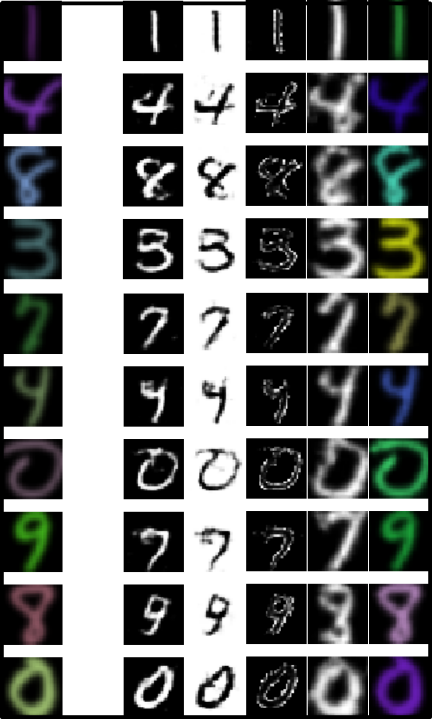}}\hspace{0.5cm}
\subfigure[]{\includegraphics[scale=.2]{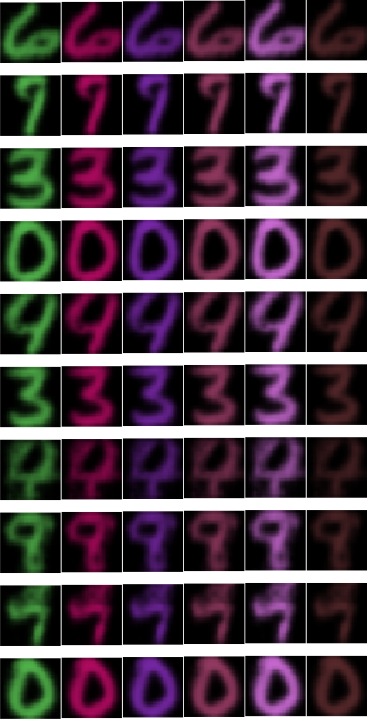}}
\vspace{-0.2cm}
\caption{Examples of transport results from MNIST (a), inverted-MNIST (b), edge-MNIST (c), USPS (d), and colorized USPS (e) to the other domains. The first column of each subfigure denotes the input image that is encoded to the latent space. Then from left to right we have the decodings to the MNIST, inverted-MNIST, edge-MNIST, USPS, and colorized USPS domains. (f) Illustration of different latent samples (rows) composed with different noise vectors (columns) decoded to the colorized USPS domains.}
\label{fig:digits}
\end{figure*}

To extend the previous results to the setting where $P_Z$ is unknown, we propose a two-step method: first estimate a suitable latent representation, and subsequently apply Algorithm \ref{alg:AE} to learn an autoencoder for each domain. 


A straight-forward approach for learning the latent distribution $P_Z$ is to train a regularized autoencoder on data from a single representative domain. However, such a representation could potentially capture variability that is specific to that one domain. 
To learn a more invariant latent representation, we propose the following extension of our autoencoder framework. The basic idea is to alternate between training multiple autoencoders until they agree on a latent representation that is effective for their respective domains. This is particularly relevant for applications to biology; for example, often one is interested in learning a latent representation that integrates all of the data modalities. 

In practice, we learned the latent distribution based on two domains $i, j \in [k]$ as follows. Let $\hat{P}_{Z_{i'}}, i' \in \{i, j\}$ denote the empirical latent distribution based on encoded data from domain $i'$, i.e. $\hat{P}_{Z_{i'}} = \pi^Z \circ E_{i'} \# P_{\X_{i'}}$. Then for domain $i$, we optimized the objective,
\begin{align*}
&\min_{E_i, D_i} \mathbb{E}_{x \sim P_{X_i}} L_1(x, D_i\circ E_i(x)) + \lambda L_2(E_i \# P_{X_i} | P_{\hat{Z}_j, N_i}),
\end{align*}
while for domain $j$, we optimized the objective,
\begin{align*}
&\min_{E_j, D_j} \mathbb{E}_{x \sim P_{X_j}} L_1(x, D_j\circ E_j(x)) + \lambda L_2(E_j \# P_{X_j} | P_{\hat{Z}_i, N_j}),
\end{align*}
where $P_{\hat{Z}_j, N_i} := \hat P_{Z_j} \otimes P_{N_i}$ and $P_{\hat{Z}_i, N_j} := \hat P_{Z_i} \otimes P_{N_j}$. The training method was identical to Algorithm \ref{alg:AE} except that we replaced the true latent distribution $P_Z$ with the empirical distributions of the encoded data.

\section{Numerical Results} \label{sec:exp}

\begin{figure*}[h]
\centering
\subfigure[]{\includegraphics[scale=.4]{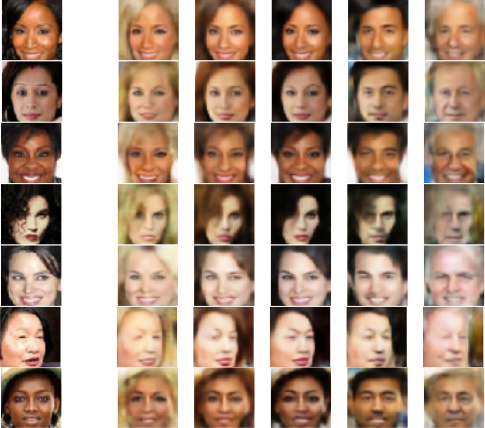}} \hspace{0.5cm}
\subfigure[]{\includegraphics[scale=.4]{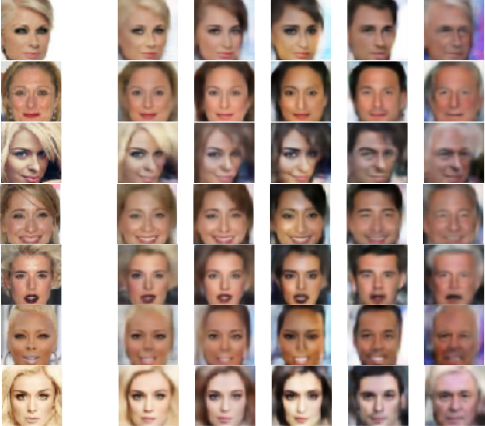}} \hspace{0.5cm}
\subfigure[]{\includegraphics[scale=.4]{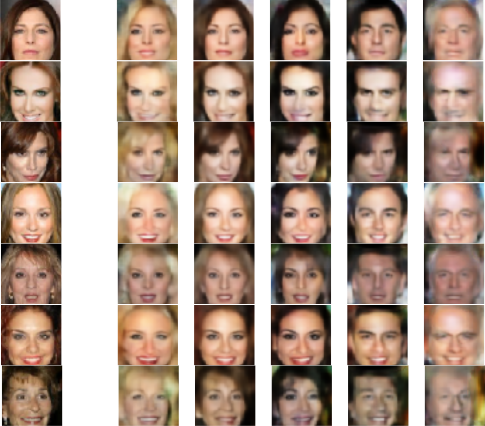}} 
\subfigure[]{\includegraphics[scale=.4]{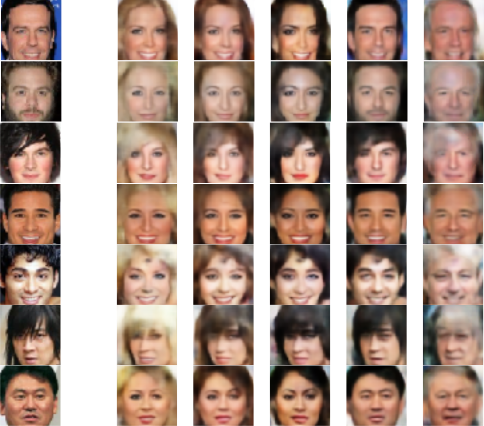}} \hspace{0.5cm}
\subfigure[]{\includegraphics[scale=.4]{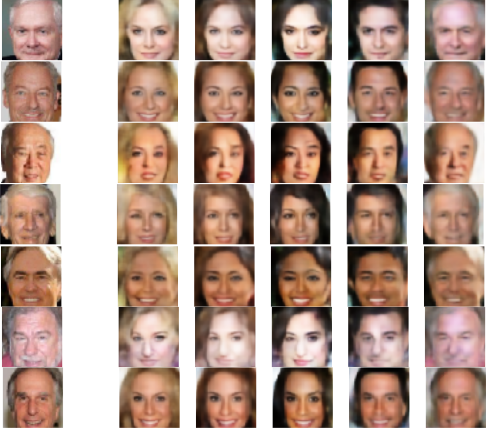}} \hspace{0.5cm}
\subfigure[]{\includegraphics[scale=.35]{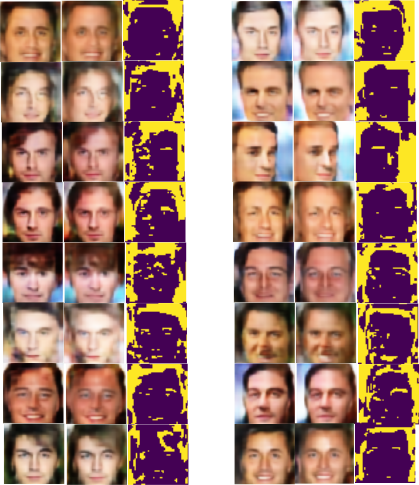}}
\vspace{-0.2cm}
\caption{Examples of transport results from black-hair females (a), blond-hair females (b), brown-hair females (c), black-hair males (d), and gray-hair males (e) to the other domains. The first column of each subfigure denotes the input image that is encoded to the latent space. Then from left to right we have the decodings to the blond-hair female, brown-hair female, black-hair female, black-hair male, and gray-hair male domains. (f) Illustration of different latent samples (rows) composed with two different noise vectors (columns) decoded to all-males domain. The last image for each face is a binary mask showing regions of the image that change the most between the two noise vectors (purple = 0, yellow = 1)}
\label{fig:celebA}
\end{figure*}

\begin{figure*}[h]
\centering
\includegraphics[scale=0.35]{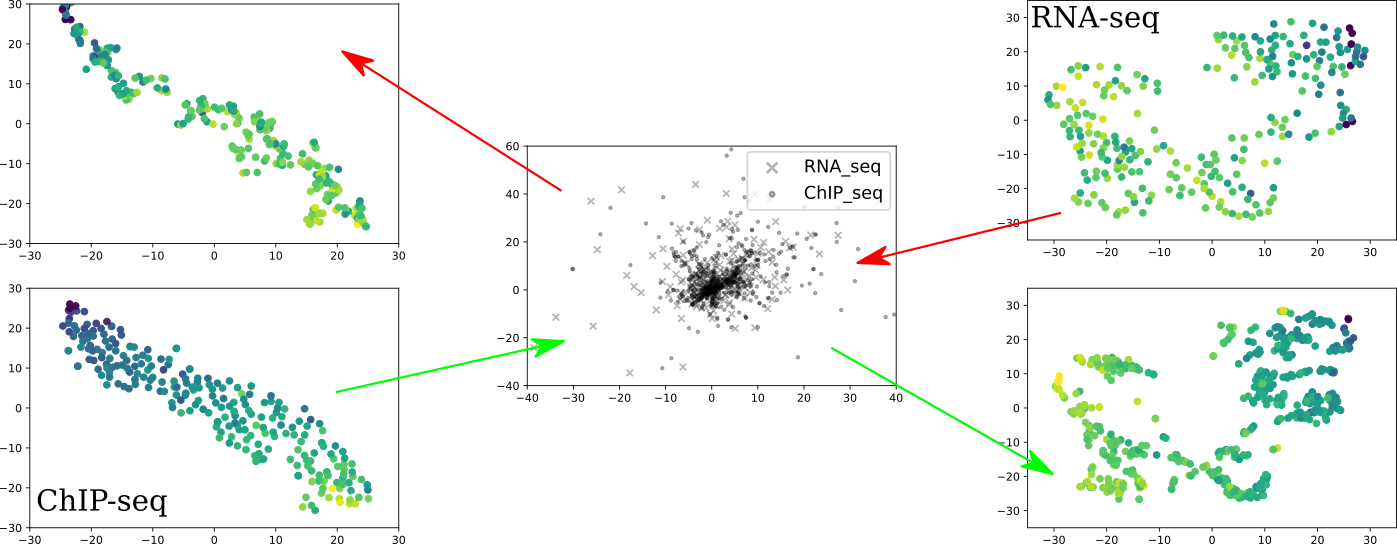}
\vspace{-0.2cm}
\caption{Visualization of transport between single-cell RNA-sequencing data and single-cell ChIP-sequencing data via a latent space representation. (Top-right) RNA-seq data; (Top-left) RNA-seq tranported to ChIP-seq domain; (Bottom-left) ChIP-seq data; (Bottom-right) ChIP-seq data transported to the RNA-seq domain; (Center) shared latent space representation. Some outliers have been cut off by the plot axes.}
\label{fig:genomics}
\end{figure*}

In this section, we show in practice how our framework can be used to perform transport between multiple domains. A key practical advantage of our framework is that once we have learned an invariant latent representation based on a couple of domains, new domains (i.e., new autoencoders) can be added sequentially to the model without retraining on the other domains.
 We demonstrate this experimentally on MNIST and CelebA image datasets as well as a genomics dataset.

\subsection{Handwritten Digits}

We first applied our algorithm towards multi-domain transport on the MNIST and USPS  handwritten digits datasets, using the following experimental procedure.

\begin{enumerate}
\itemsep0em 
\item Estimate the latent representation of handwritten digits by training autoencoders on the MNIST dataset and a synthetic MNIST variant (inverted-MNIST) as described in Section \ref{sec_PZ_unknown}.
\item Using the encoded MNIST dataset as a proxy for the latent distribution, train autoencoders separately on the following datasets as described in Section \ref{sec_PZ_known}:
\begin{enumerate}
\itemsep0em 
\item A synthetic MNIST variant (edge-MNIST)
\item USPS digits
\item Colorized USPS digits
\end{enumerate}
\item Compose autoencoders from Steps 1-2 for multi-domain transport
\end{enumerate}

The composition of encoders and decoders from 5 autoencders results in a total of 20 transport maps (or 10 pairwise maps) between the 5 domains. The results of the transport maps are shown in Figure \ref{fig:digits}. Importantly, we \emph{did not} tune or retrain any of the existing autoencoders when sequentially adding the new domains under Step 2 to the model. This suggests that Algorithm \ref{alg:AE} is successful at encoding the data from different domains to the same latent representation.

With exception of the colorized USPS digits, all of the digits from different domains can reasonably be assumed to be generated \emph{deterministically} from the latent space of handwritten digits. For the colorized USPS digits, one would expect color to be encoded \emph{stochastically} from the noise variable $N_i$ rather than the shared latent variable $Z$. Our autoencoder for the colorized USPS digits is designed to capture this generative process (Figure \ref{fig:digits}f); when decoding random samples from the latent and noise distribution to colorized USPS digits, the shape and color of the digits varies with the latent variable and noise variable respectively.

We also attempted this experiment using a vanilla autoencoder on the MNIST dataset to learn the latent representation of the digits in Step~1. While the transport results were acceptable between the different MNIST variants, the transport between USPS and MNIST fared significantly worse. This indicates that the latent representation learned by alternating between two domains is more invariant to the different handwritten digit domains, and further motivates the importance of learning a quality latent representation in future work.

\subsection{Celebrity Faces}

Subsequently we used our framework to perform multi-domain translation on the CelebA \cite{liu2015faceattributes} celebrity faces dataset. The experiments were performed using a similar procedure as for the handwritten digits:

\begin{enumerate}
\itemsep0em 
\item Learn a latent representation for faces by training paired autoencoders on \emph{black-haired females} and \emph{blond-haired females} as described in Section \ref{sec_PZ_unknown}.
\item Using the encoded images of black-haired and blond-haired females as a proxy for the latent distribution, train autoencoders separately on the following datasets as described in Section \ref{sec_PZ_known}:
\begin{enumerate}
\itemsep0em 
\item \emph{brown-haired females}
\item \emph{black-haired males}
\item \emph{gray-haired males}
\item \emph{all males}
\end{enumerate}
\item Compose autoencoders from Steps (1-2) for transport
\end{enumerate}

The composition of encoders and decoders from 6 autoencoders results in a total of 30 transport maps (or 15 pairwise maps) between the 6 domains. Results of the transport are shown for a subset of these maps in Figure \ref{fig:celebA}a-e. Once again, we \emph{did not} tune or retrain the existing autoencoders when adding the domains from Step 2 to the model. In fact, none of the new autoencoders from Step 2 were exposed to the data from other new domains during training; the only data that was used was the latent encodings from the initial two domains. This suggests that Algorithm \ref{alg:AE} is successful at encoding the data from different domains to the same latent representation.

The autoencoder from Step 2d must learn the faces of males with different hair colors, while the latent representation in Step 1 was learned from domains of female faces with the same hair color. Therefore one would expect hair color to be encoded \emph{stochastically} by the noise variable $N_i$. We found qualitative evidence of the noise variable affecting the darkness of hair when using the autoencoder to decode random samples from the latent distribution with select noise vectors (Figure \ref{fig:celebA}f).

{\bf Progressive Training.} Training the autoencoders on the CelebA faces was significantly more challenging than handwritten digits. To guide and stabilize training, we employed a strategy of \emph{progressive training} inspired by \cite{karras2017progressive} for GANs. We first trained the autoencoders in Step 1 on lower resolution (8 x 8) images and gradually progressed to higher resolution (64 x 64) images, adding new convolutional layers to both the encoder and the decoder components. We then performed the same procedure in Step 2, using the final encodings from Step 1. The discriminator was not affected as it continued to train in the latent space.

\subsection{Genomics}

Datasets generated by different single-cell experimental methods give only a partial view of the cell state. Processes such as development and disease progression are driven by the interplay between many components of a cell (e.g. expression of key genes, chromatin conformation, epigenetic modifications, etc.) but each type of experiment only produces data from one domain. Currently, methods are lacking for integrating and translating between different biological domains.

A significant difference between image-based and genomics-based translation is the role that neural architecture plays in enforcing correspondence between translated elements. In image-to-image translation, the structure of natural images is preserved in the convolutional structure of the neural network, which is biased by design towards the desired coupling between images. On the other hand, natural structure in genomics data, such as single-cell gene expression, is based on the (unknown) 3D associations between different loci on the DNA and is not naturally preserved in the neural network structure. As a result, some high-level form of supervision (e.g.~in the form of generic class labels) may be necessary to preserve the correct structure during transport.

We used our autoencoder framework to perform translation between unpaired single-cell RNA-sequencing data (RNA-seq, \citet{kolodziejczyk2015single}) and single-cell ChIP-sequencing data processed into chromatin signatures (ChIP-seq, \citet{rotem2015single}) from mouse embryonic stem cells. When performing this experiment without using class labels, the distributions of the RNA-seq and ChIP-seq data appeared to be Gaussian in the latent space, which allowed for arbitrary alignment of the transported distributions with the real domain distributions. We found that providing binary class labels based on the pluripotency marker Oct4 (Pou5f1) to the discriminator in Algorithm \ref{alg:AE} 
was sufficient to orient the encoded distributions of the RNA-seq and ChIP-seq data in the latent space. This resulted in translated distributions that were well-aligned with the true distributions based on the full range of pluripotency (Figure \ref{fig:genomics}). Our framework can be used to make predictions with regards to how different biological domains are related, which can subsequently be validated by experimentation.

\section{Discussion}

We proposed a novel framework for learning transport maps between multiple domains, based on leveraging a shared latent representation to decompose the multi-domain translation model between $k$ domains into $k$ \emph{uncoupled} autoencoders. These autoencoders are trained separately and composed to translate between different pairs of domains. Compared to existing frameworks, our model offers several advantages such as being flexible (i.e., autoencoders can be trained separately on different domains) and efficient (i.e., only one discriminator is required in the latent space; distribution matching is done in the latent space rather than the original space). Future work will focus on better methods for learning invariant representations in order to improve data integration as well as generative quality.

\pagebreak
\section*{Acknowledgements}
 Discussions with G.V. Shivashankar on establishing the link between genome architecture and gene expression motivated this work on the development of methods for integrating and translating different data modalities. Karren Yang was partially supported by an NSF Graduate Fellowship. Caroline Uhler was partially supported by NSF (DMS-1651995), ONR (N00014-17-1-2147 and N00014-18-1-2765), IBM, and a Sloan Fellowship. The Titan Xp used for this research was donated by the NVIDIA Corporation.

\bibliographystyle{icml2019}
\bibliography{icml_2018}

\thispagestyle{empty}
\onecolumn
\icmltitle{Multi-Domain Translation
	by Learning Uncoupled Autoencoders (Supplementary Material)}
\appendix
\allowdisplaybreaks[1]

\section{Proofs from Main Text}

For convenience we repeat the main assumptions, definitions, and results here:
\begin{assumption} \label{app-ass:sem} The random variables $X_1, \dots, X_k$ are generated by the following causal generative mechanism:
\begin{equation*}
X_i = f_i(Z, N_i), \quad \forall i \in [k],
\end{equation*}
where $f_i$, $i\in[k]$ are injective functions, $Z$ is a latent variable with distribution $P_Z$, and $N_i$, $i\in[k]$, are independent noise variables with distribution $P_{N_i}$. 
\end{assumption}

\begin{definition}[Path consistency]
Let $(i_1, \cdots, i_{\ell})$ denote a sequence of domains. A tuple of autoencoders $(E_i, D_i)_{i \in [k]}$ is \emph{path-consistent} if for every finite sequence $(i_1, \cdots, i_{\ell})$,
\begin{align*}
\int_{\X_{i_{\ell-1}}} ... \int_{\X_{i_2}}
&\prod_{j'=1}^{\ell-1} q_{X_{i_{j'+1}}|X_{i_{j'}}}(x_{j'+1}|x_{j'})
 dx_2...dx_{\ell-1} \\
&= q_{X_{i_\ell}|X_{i_1}}(x_{\ell}|x_1)
\end{align*}
\end{definition}
Path-consistency implies that any sequence of encodings and decodings starting in a domain $i_1$ and ending in a domain $i_{\ell}$ induces the same conditional distribution.

\begin{definition}[Global consistency]
Let $Q^{(i)}$ be the joint distribution over $X_1, \cdots, X_k$ with density given by
\begin{align*}
q^{(i)}(\bm{x}) := \int_{\Z} \prod_{j \neq i} q_{X_j|Z}(x_j|z)q_{Z|X_i}(z|x_i) p_{X_i}(x_i) dz.
\end{align*}
A tuple of autoencoders $(E_i, D_i)_{i \in [k]}$ satisfies \emph{global consistency} if $Q^{(1)} = Q^{(2)} = \cdots = Q^{(k)}$. 
\end{definition}

Global consistency means that the joint probability distribution generated by encoding a domain $X_i$ using $E_i$ and decoding the resulting latent variable to all other domains $j\in[k]$, $j\neq i$ using $D_j$ is equivalent for any source domain $i\in [k]$.

\begin{proposition} \label{app-the:sym}
Under Assumption \ref{app-ass:sem} every optimal solution of encoder-decoder tuples $(E_i,D_i)_{i\in[k]}$ to the optimization problem 
\begin{equation}
	\label{app-eq_supp}
\min_{(E_i,D_i)_{i\in [k]}}\mathbb{E}_{x \sim P_{X_i}} \left[ L_1(x, D_i\circ E_i(x)) + \lambda L_2(E_i \# P_{X_i} | P_{Z, N_i}) \right]
\end{equation} 
 satisfies path and global consistency.
\end{proposition}

To prove this we first introduce some helpful lemmas.

\begin{lemma} The optimal value of the optimization problem in (\ref{app-eq_supp}) is zero. 
\end{lemma}

\begin{proof}
We choose $(E_i, D_i)$ as follows: Take $D_i = f_i$ and note that $P_{X_i} = D_i \# (P_{Z} \otimes P_{N_i})$. Now we construct $E_i$ such that $E_i \# P_{X_i} = P_{Z} \otimes P_{N_i}$ and $D_i \circ E_i(x) = x$ almost surely with respect to $P_{\X_i}$. To do this, first note that there exists a subset $A \subseteq \X_i$ such that the restriction $\tilde{D}_i: \Z \times \N_i \rightarrow A$ is surjective and $P_{X_i}(A)=1$. Since $\tilde{D}_i$ is injective by hypothesis, it follows that $\tilde{D}_i: \Z \times \N_i \rightarrow A$ is a bijective map with an inverse $\tilde{D}_i^{-1}: A \rightarrow \Z \times \N_i$. Let $E_i: \X \rightarrow \Z \times \N_i$ be any function such that $E_i(x) = \tilde{D}_i^{-1}(x), \forall x \in A$. It follows that $D_i(E_i(x))= D_i(D^{-1}_i(x))=x, \forall x \in A$ and that for all Borel sets $B \subseteq \Z \times \N_i$,
$$E_i \# P_{X_i}(B) = D_i^{-1} \# P_{X_i}(B) = P_{X_i} (D_i(B)) = D_i \# (P_{Z} \otimes P_{N_i}) (D_i(B)) = P_{Z} \otimes P_{N_i} (B).$$
Observe that this choice of $(E_i, D_i)$ achieves zero loss for the objective in (\ref{app-eq_supp}), which is the smallest possible value given the non-negativity of $L_1$ and $L_2$.
\end{proof}

\begin{lemma} 
	\label{app-lemma_1}
	For any tuple of optimal solutions $(E_i, D_i)_{i \in [k]}$ to (\ref{app-eq_supp}), there exists $A_i \subseteq \X_i$ and $B_i \subseteq \Z \times \N_i$, $i\in[k]$ such that 
	\begin{enumerate}
	\item[(i)] $P_{X_i}(A) = 1$, 
	\item[(ii)] $P_Z \otimes P_{N_i}(B_i) = 1$, 
	\item[(iii)] the restrictions $E_i: A_i \rightarrow B_i$ and $D_i: B_i \rightarrow A_i$ are bijective with $E_i = D_i^{-1}$, 
	\item[(iv)] $\pi^{Z} B_1 = \cdots = \pi^{Z} B_k$, where $\pi^Z: (z, n) \mapsto z$ is the projection to $Z$.
	\end{enumerate}
\end{lemma}

\begin{proof}
We construct $A_i$ and $B_i$ as follows. From the previous lemma, it follows that for all $i \in [k]$, any optimal solution $(E_i, D_i)$ to (\ref{app-eq_supp}) satisfies
$$\mathbb{E}_{x \sim P_{X_i}} L_1(D_i(E_i(x)),x) = 0.$$

Since $L_1$ is non-negative and equal to zero if and only if $D_i(E_i(x)) = x$, this implies that there exists $A_i \subseteq \X_i$ such that $P_{X_i}(A_i)=1$ and $D_i(E_i(x))=x$ for all $x \in A_i$. Consider the restriction $E_i: A_i \rightarrow \Z \times \N_i$; $E_i$ has a left inverse $D_i$ over this domain which implies that $E_i$ is injective. Note that $E_i \# P_{X_i} = P_{Z} \otimes P_{N_i} \implies  (\pi^Z \circ E_i) \# P_{X_i} = P_{Z}$. This implies there exists $B \subseteq \Z$ such that for all $i \in [k]$, $\pi^Z \circ E_i: A_i \rightarrow B$ is surjective and $P_{Z}(B) = 1$. Additionally, since $E_i \# P_{X_i} = P_{Z} \otimes P_{N_i}$, there exists $B_i \subseteq B \times \N_i$ with $\pi^{Z} B_i = B$ such that $P_Z \otimes P_{N_i}(B_i) = 1$ and the restriction $E_i: A_i \rightarrow B_i$ is surjective. It follows that $E_i: A_i \rightarrow B_i$ is bijective with inverse $D_i: B_i \rightarrow A_i$. This concludes the proof.
\end{proof}

This lemma implies that for any optimal autoencoder $(E_i, D_i)$ the encoder and decoder are the inverses of each other when the domains and codomains are restricted to $A_i, B_i$. Since all the data falls into these sets with probability one, in the subsequent discussion, for any given optimal $(E_i, D_i)$, we can restrict $\X_i = A_i$, $\Z = B$, and $\Z \times \N_i = B_i$ and assume $E_i = D_i^{-1}$ without loss of generality.

\begin{proof}[Proof of Proposition \ref{app-the:sym}]
It suffices to prove path-consistency for the case of $\ell=3$; for longer sequences the result can be proven by induction. Without loss of generality, consider the sequence $(1,2,3)$. Note that
\begin{align*}
\int_{\X_2} q(x_3|x_2)q(x_2|x_1) dx_2
&= \int_{\X_2} \left[\int_\Z q(x_3|z')q(z'|x_2) dz' \right] \left[ \int_\Z q(x_2|z) q(z|x_1) dz \right] dx_2\\
&= \int_{\X_2} \left[\int_\Z q(x_3|z') \delta_{\pi^Z(E_2(x_2))=z'} dz' \right]
\left[ \int_\Z \int_{\N_2} \delta_{D_2(z,n)=x_2} dN_2(n) q(z|x_1) dz \right] dx_2 \\
&= \int_\Z \int_\Z \int_{\X_2} \int_{\N_2} q(x_3|z') \delta_{\pi^Z(E_2(x_2))=z'} \delta_{D_2(z,n)=x_2} q(z|x_1) dN_2(n) dx_2 dz dz'\\
&= \int_\Z \int_{\Z} \int_{\N_2} q(x_3|z') \delta_{\pi^Z(E_2(D_2(z,n)))=z'} q(z|x_1)dN_2(n) dz dz'\\
&= \int_\Z \int_{\Z} \int_{\N_2} q(x_3|z') \delta_{\pi^Z(z,n)=z'} q(z|x_1)dN_2(n) dz dz'\\
&= \int_\Z \int_{\Z} q(x_3|z') \delta_{z=z'} q(z|x_1) dz dz'\\
&= \int_\Z q(x_3|z') q(z'|x_1) dz'\\
&:= q(x_3|x_1),
\end{align*}
which proves path consistency.

For global consistency, we introduce the following notation: for any conditional probability distribution $Q_{Y_1|Y_2}$ and any probability distribution $P_{Y_2}$ we denote their joint distribution by $Q_{Y_1|Y_2} \otimes P_{Y_2}$. We first show that the joint distributions $Q_{X_i|Z, N_i} \otimes P_Z \otimes P_{N_i}$ and $Q_{Z, N_i|X_i} \otimes P_{X_i}$ are equal over any subsets $S \subseteq \X_i$, $S' \subseteq \Z \times \N_i$:

\begin{align*}
Q_{X_i|Z, N_i} \otimes P_Z \otimes P_{N_i} (S, S') &= \mathbbm{1}_{D(Z, N) \in S} P_Z \otimes P_{N_i} (S')\\
&= P_Z(S' \cup D^{-1}(S))\\
&= P_X(E^{-1}(S' \cap D^{-1}(S)))\\
&= P_X(E^{-1}(S') \cap E^{-1} \circ D^{-1} (S))\\
&= P_X(E^{-1}(S') \cap D \circ D^{-1} (S))\\
&= P_X(E^{-1}(S') \cap S)\\
&= \mathbbm{1}_{E(X) \in S'} P_{X_i} (S)\\
&= Q_{Z, N_i|X_i} \otimes P_{X_i}(S, S')
\end{align*}
By marginalizing out $N_i$ (i.e. taking $S' = A \times \N_i$ for some $A \subseteq \Z$), it follows that
$Q_{X_i|Z} \otimes P_Z =Q_{Z|X_i} \otimes P_{X_i}$. Based on this result, we have for any $i \in [k]$:

\begin{align*}
\bigotimes_{j \in [k]} Q_{X_j | Z} \otimes P_Z
= \bigotimes_{j \in [k], j \neq i} Q_{X_j | Z} \otimes Q_{X_i|Z} \otimes P_Z 
= \bigotimes_{j \in [k], j \neq i} Q_{X_j | Z} \otimes Q_{Z| X_i} \otimes P_{X_i} = Q^{(i)}
\end{align*}
which gives the desired result.
\end{proof}

In general, the optimal solutions to the optimization problem in  (\ref{app-eq_supp}) are not unique because there can potentially be multiple ways to map between $Z$ and $X_i$. However, we can guarantee that the true probabilistic coupling under Assumption \ref{app-ass:sem} is represented by some solution in the optimal set. This is formalized in the following proposition.

\begin{proposition}[Completeness] \label{app-the:exp}
For any probabilistic coupling $P_{{\bf X}}$ satisfying Assumption~\ref{app-ass:sem}, there exists a tuple of autoencoders $(E_i, D_i)_{i \in [k]}$ solving (\ref{app-eq_supp}) such that $P_{\bf X} = Q_i$ for all $i \in [k]$, where $Q_i$ is defined as in Proposition \ref{app-the:sym}.
\end{proposition}

\begin{proof}
The construction of such an optimal solution is given by Lemma \ref{app-lemma_1}. Combining this with the previous proposition yields the desired result.
\end{proof}

\begin{proposition} Let $Q_{X_{i \rightarrow j}}$ denote the distribution of $X_{i \rightarrow j}:= D_j(\pi^Z(E_i(X_i)), N_j)$. If the decoder $D_j$ is $\gamma_j$-Lipschitz, then the 1-Wasserstein distance $W(Q_{X_{i \rightarrow j}}, P_{X_j})$ satisfies
\begin{align*}
W(Q_{X_{i \rightarrow j}}, P_{X_j}) &\leq \gamma_j W(E_i \# P_{X_i}, P_Z \otimes P_{N_i}) \\
&+ \gamma_j W(P_Z \otimes P_{N_j}, E_j \# P_{X_j}) \\
&+ \mathbb{E}_{x \sim P_{X_j}} L_1(x, D_j\circ E_j(x)).
\end{align*}
\end{proposition}

\begin{proof}
The proof is obtained using the triangle inequality and Lipschitz property of Wasserstein metrics, which can be found in \cite{patrini2018sinkhorn}. By applying these properties we get:
\begin{align*}
W(Q_{X_{i \rightarrow j}}, P_{X_j}) 
&\leq W(Q_{X_{i \rightarrow j}}, Q_{X_{j \rightarrow j}}) 
+ W(Q_{X_{j \rightarrow j}}, P_{X_j})\\
&\leq \gamma_j W((\pi^\Z \circ E_i \# P_{X_i}) \otimes P_{N_j}, E_j \# P_{X_j})
+ W(Q_{X_{j \rightarrow j}}, P_{X_j})\\
&\leq \gamma_j \left[ W((\pi^\Z \circ E_i \# P_{X_i}) \otimes P_{N_j}, P_Z \otimes P_{N_j}) + W(P_Z \otimes P_{N_j}, E_j \# P_{X_j}) \right] 
+ W(Q_{X_{j \rightarrow j}}, P_{X_j})\\
&\leq \gamma_j \left[ W(E_i \# P_{X_i}, P_Z \otimes P_{N_i}) + W(P_Z \otimes P_{N_j}, E_j \# P_{X_j}) \right] 
+ W(Q_{X_{j \rightarrow j}}, P_{X_j})\\
&\leq \gamma_j \left[ W(E_i \# P_{X_i}, P_Z \otimes P_{N_i}) + W(P_Z \otimes P_{N_j}, E_j \# P_{X_j}) \right] 
+ \mathbb{E}_{x \sim P_{X_j}} L_1(x, D_j\circ E_j(x)),
\end{align*}
which concludes the proof.
\end{proof}

\section{Supplement to Numerical Experiments.}

{\bf Genomics Datasets.} To illustrate a biological application of the translation model from our genomics experiment, we performed an analysis of correspondence between the gene expression (from RNA-seq data) and chromatin signatures (from ChIP-seq data). Specifically, we filtered for genes and ChIP signatures that are significantly correlated with the same latent dimension. Figure \ref{app-fig:genomics2}a shows some examples of cellular processes based on the GOrilla gene ontology \cite{eden2009gorilla} analysis of genes correlated with the same latent dimension as the chromatin signatures in Figure \ref{app-fig:genomics2}b. A few notable correlations that were statistically significant and are corroborated by biological evidence include the following:

\begin{itemize}
\item There were several cellular processes associated with pluripotency with significant p-values: stem cell population maintenance ($p = 1.05 \times 10^{-7}$), stem cell differentiation ($p = 2.13 \times 10^{-7}$), somatic stem cell population maintenance ($p = 4.35 \times 10^{-4}$), and stem cell division ($p = 4.37 \times 10^{-4}$). These correspond with ChIP signatures such as Oct4, Sox2, and Nanog, which are known to regulate pluripotency \cite{kashyap2009regulation}.
\item Several processes associated with histone modification with significant p-values, including and not limited to: regulation of histone modification ($p=1.59 \times 10^{-7}$), histone acetylation ($p=5.16\times 10^{-7}$), regulation of histone methylation ($p=3.38 \times 10^{-5}$). These correspond with ChIP signatures based on histone markers such as H3K4me1, H3K4me3, H3K9ac, etc.
\item Processes related to chromosomal regulation and cell division, including and not limited to: protein localization to chromosomes ($p=2.77 \times 10^{-6}$), protein localization to the centromeric region of chromosomes ($p=3.83 \times 10^{-4}$), cell division $p=4.02 \times 10^{-7}$. These correspond with ChIP signatures NIPBL, Smc1, Smc3, etc. which are known to regulate chromosomes during cell division \cite{zuin2014cohesin, eijpe2000association}.
\end{itemize}

\begin{figure}
\centering
\subfigure[]{\includegraphics[scale=.3]{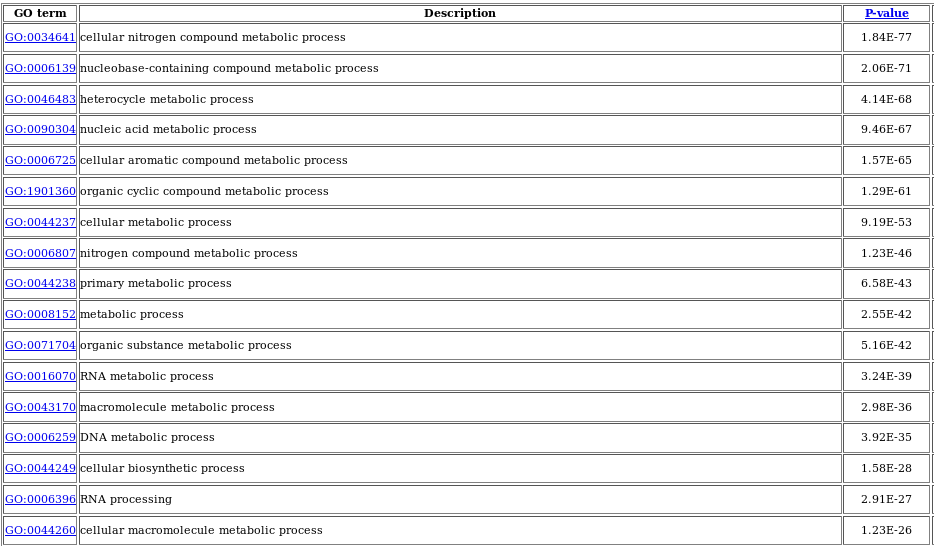}} \hspace{0.7cm}
\subfigure[]{\includegraphics[scale=.2]{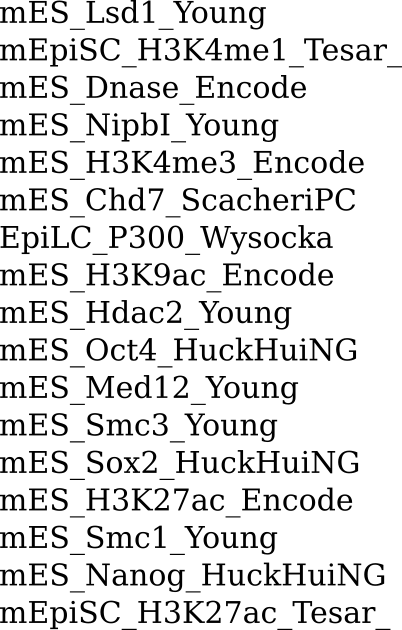}}
\caption{(a) Examples of gene ontology terms strongly correlated with the latent dimension based on the GOrilla analysis tool \cite{eden2009gorilla} (b) Examples of chromatin signatures strongly correlated with the latent dimension; descriptions can be found in \cite{rotem2015single}}
\label{app-fig:genomics2}
\end{figure}



\end{document}